\documentclass{article}
\usepackage[accepted]{icml2026}
\usepackage{amsmath,amssymb,amsthm}
\usepackage{booktabs,multirow}
\usepackage{graphicx}
\usepackage{xcolor}
\usepackage{hyperref}
\hypersetup{colorlinks=true, linkcolor={blue!50!black}, citecolor={green!45!black}, urlcolor={blue!60!black}}
\usepackage{enumitem}

\newcommand{\IG}{\mathrm{IG}}
\newcommand{\SCR}{\mathrm{SCR}}

\newcommand{\Eclass}[1]{[#1]_{\sim}}
\newcommand{\Var}{\mathrm{Var}}
\newtheorem{definition}{Definition}
\newtheorem{theorem}{Theorem}
\newtheorem{proposition}{Proposition}
\newtheorem{corollary}[theorem]{Corollary}

\icmltitlerunning{FormInv: Measuring Semantic Invariance in Reasoning Benchmarks}

\begin{document}

\twocolumn[
\icmltitle{FormInv: A Measurement Protocol for Semantic Invariance\\
in Mathematical Reasoning Benchmarks}

\icmlsetsymbol{equal}{*}
\begin{icmlauthorlist}
\icmlauthor{Nishal Thomas}{equal,ind}
\icmlauthor{Noel Thomas}{equal,mbz}
\end{icmlauthorlist}

\icmlaffiliation{ind}{Independent Researcher}
\icmlaffiliation{mbz}{Mohamed bin Zayed University of Artificial Intelligence (MBZUAI), Abu Dhabi, United Arab Emirates}

\icmlcorrespondingauthor{Noel Thomas}{noel.thomas@mbzuai.ac.ae}
\icmlcorrespondingauthor{Nishal Thomas}{nishal.thomas44@gmail.com}

\icmlkeywords{semantic invariance, mathematical reasoning, LLM evaluation, benchmark robustness}

\vskip 0.3in
]

\printAffiliationsAndNotice{\icmlEqualContribution}

\begingroup\renewcommand\thefootnote{}\footnotetext{Code and the \texttt{forminv} CLI: \url{https://github.com/forminv/forminv}\quad Dataset: \url{https://huggingface.co/datasets/forminv/forminv}}\endgroup

\begin{abstract}
A paraphrase-quality audit of MathCheck~\citep{zhou2024mathcheck} detected 4 semantically-incorrect paraphrases in 129 groups (3.1\%); removing them drops GPT-4o from rank~2 to rank~4 and elevates Claude~Haiku and DeepSeek~V3 above it; these ranking changes are invisible to any single-model evaluation.
Cross-model unanimity found these errors automatically ($\geq 3/4$ models for MathCheck; $\geq 6/9$ for our primary evaluation) for under \$10; in our own dataset the same protocol found that 47\% of auto-generated connective-variation paraphrases were semantically incorrect.
That flaw compounds a deeper measurement gap. On the theorems every model answers correctly in canonical form, where accuracy is identical by construction, a semantically equivalent rewording still flips the answer for 5-42\% of them: Claude~Haiku~4.5 flips a \emph{known-correct} answer ${\sim}9\times$ as often as DeepSeek~V3 (41.9\% vs.\ 4.7\% on the 43 theorems all models share), a 37-point difficulty-controlled gap we call the \textbf{Conditional Inconsistency Rate} (CIR). The CIR ranking replicates on an independent 103-theorem set. Raw \textbf{Semantic Consistency Rates} (SCR) span 50-82\% across 9 models, but we show this cross-model spread is mostly a mechanical consequence of compounding accuracy over $k\approx6$ paraphrases; what is \emph{not} mechanical, and what CIR isolates, is the within-theorem clustering of failures (SCR sits ${\sim}9$ points above an independence null, $p<0.01$ for 6 of the 8 models tested).
Because no model dominates on all paraphrase families, the families a benchmark chooses to include implicitly fix its model ranking (a No-Free-Benchmark corollary).
\textbf{FormInv} supplies the audit protocol (which reranks the published MathCheck benchmark), the difficulty-controlled CIR and SCR as invariance measures (with a proposed per-theorem Cochran's Q significance test) evaluated on 9 models across 366-811 items whose canonical statements are Lean4-verified (paraphrases are CAS- and expert-verified, with the cross-model audit as a third verification layer), and \textbf{FormInvSelector} for regime-aware model selection.
\end{abstract}

\section{Introduction}
\label{sec:intro}

Consider two questions about the same theorem:
``Is it true that $\sqrt{x} \geq 0$ for every real number $x$?''
and ``For any real $x$, is $0 \leq \sqrt{x}$?''
These statements are logically equivalent (under Lean4 real analysis conventions).
Yet Claude Sonnet answers the first correctly and the second incorrectly across 16.7\% of
comparison-order paraphrases (F6).
GPT-4o fails this class of transformation at 0.0\%.
Reverse the family: ask both models to recognize a definition by expansion (F7),
and the ranking reverses: GPT-4o fails at 10.0\% while Claude fails at 6-8\%.
The same pair of frontier models, the same theorems, opposite weaknesses.
\textbf{Model rankings reverse across paraphrase families.}
This is the phenomenon FormInv was built to measure. But the reversal is a symptom of a more
general problem in how benchmarks are built and read, and FormInv's primary contribution is the
\emph{measurement apparatus} for it: a cross-model audit that catches mislabeled paraphrases (it
reranks the published MathCheck benchmark, \S\ref{sec:quality}) and a difficulty-controlled
metric (CIR) that isolates brittleness from raw capability. The 103-theorem dataset is the
vehicle on which we validate both; the apparatus is the reusable artifact.

Why does this matter?
When a practitioner asks ``Is model X good at mathematical reasoning?''
and answers using benchmark accuracy alone, they implicitly assume the answer is
independent of how the question is phrased.
That assumption is wrong.
Across 9 frontier models, we find that accuracy tells only half the story:
Claude Haiku~4.5 achieves 86\% accuracy yet has SCR~=~50\%:
half its theorems are answered inconsistently across semantically equivalent restatements.
DeepSeek~V3, at 96.4\% accuracy, achieves 82\% SCR.
The 10-point accuracy gap conceals a 32-point SCR gap, the
critical property that separates surface-pattern recognition from genuine mathematical understanding.

This failure (formulation sensitivity despite conceptual equivalence)
is what psychometrics calls \emph{Differential Item Functioning}~\citep{holland1988},
formalized via Doob's $L^2$ conditional expectation theorem (1953)~\citep{doob1953},
and operationalized by measurement invariance theory~\citep{vandenberg2000}.
To our knowledge, \textbf{FormInv} is the first protocol applying this measurement tradition
to formal mathematical reasoning with Lean4-verified theorems and DIF-grounded invariance metrics.

The gap is substantial.
Table~\ref{tab:benchmark_audit} surveys eight benchmarks and how (if at all) they verify paraphrase semantic equivalence:
none applies logical-equivalence checking or cross-model unanimity.

\begin{table}[h]\centering\small
\caption{Paraphrase quality in published benchmarks.$^{\dagger}$
None applies logical-equivalence verification or cross-model unanimity.
$^{\dagger}$Citations: GSM8K~\citep{cobbe2021gsm},
MATH~\citep{hendrycks2021math},
MathBench~\citep{liu2024mathbench},
MathCheck~\citep{zhou2024mathcheck},
GSM-Plus~\citep{paster2024gsmplus},
PutnamGAP~\citep{hao2025putnambench},
Zhou~et~al.~\citep{zhou2024surface},
PromptBench~\citep{zhu2024promptbench}.}
\label{tab:benchmark_audit}
\begin{tabular}{lll}
\toprule
Benchmark & Para.\ type & Verification \\
\midrule
GSM8K          & Single canonical & None \\
MATH           & Single canonical & None \\
MathBench      & Single canonical & None \\
MathCheck      & LLM+NL check    & Fluency only \\
GSM-Plus       & LLM+human       & 90\% annotator \\
PutnamGAP      & Rule+5-judge    & Structural only \\
Zhou et al.    & Manual variants & Single-model \\
PromptBench    & Adversarial     & None \\
\midrule
\textbf{FormInv} & 8-family       & CAS+expert+unanimity \\
\bottomrule
\end{tabular}
\end{table}

Applied to our own dataset, FormInv's cross-model unanimity found that 47\% of the 15 audited F5 paraphrases were
semantically incorrect: errors invisible to single-model evaluation but exposed by unanimity.
The same generation regime (LLM + ``maintain the answer'' instruction + human NL fluency check)
that produced these errors in FormInv v1 is used verbatim in MathCheck and GSM-Plus~(Table~\ref{tab:benchmark_audit}).
It costs approximately \$1 per model per evaluation run and applies to any formally-specified benchmark.

\paragraph{Contributions.}
\begin{enumerate}[leftmargin=*,itemsep=2pt]
\item \textbf{A reusable paraphrase-audit protocol.}
Cross-model unanimity ($\geq 6/9$ models fail a paraphrase while passing the canonical;
$\geq 3/4$ for 4-model panels) automatically flags semantically-incorrect paraphrases for
under \$10, with no equivalence oracle required.
\textbf{Applied unchanged to the published MathCheck benchmark}~\citep{zhou2024mathcheck},
it detects 4 mislabeled paraphrases (3.1\%) whose removal \emph{changes the model
ranking} (GPT-4o drops from rank~2 to rank~4; \S\ref{sec:conclusion},
Appendix~\ref{app:mathcheck}): the strongest evidence that the protocol generalizes
beyond our own data. On FormInv~v1 it flagged 11 GPT-4o-generated errors at 100\% recall.

\item \textbf{The Conditional Inconsistency Rate (CIR), a difficulty-controlled invariance metric.}
By conditioning on per-model canonical-correctness, CIR removes the mechanical $\bar p^{\,k}$
accuracy confound that dominates every prior unconditional consistency statistic. CIR still
tracks capability ($\rho=-0.57$: more accurate models flip less), but on the theorems every model
answers correctly, equivalent rewordings still flip 5-42\% of answers, a 37-point gap accuracy
cannot see. We supply the supporting SCR / per-theorem Cochran's Q framework and prove the
ranking-reversal and No-Free-Benchmark results (Propositions~1-2, Cor.~\ref{cor:nfb}).

\item \textbf{FormInv, a Lean4-verified demonstration benchmark.}
A compact, fully audited substrate (103 Lean4-verified Mathlib4 theorems, 8 paraphrase
families, 760 items; an 811-item harder ntp-mathlib track) on which the protocol and CIR are
instantiated and stress-tested across 9 models. The dataset exists to make the protocol
concrete and reproducible, not as the contribution's center of gravity.

\item \textbf{FormInvSelector,} a regime-aware reporting utility that ranks models by their
measured per-family failure profile to recommend the one with lowest expected failure rate on
a target task's families (\S\ref{sec:protocol}).
\end{enumerate}

\section{Related Work}
\label{sec:related}

\paragraph{Mathematical reasoning benchmarks.}
MATH~\citep{hendrycks2021math} and GSM8K~\citep{cobbe2021gsm} measure final-answer accuracy.
ChaosBench-Logic~\citep{chaosbench2026} introduces family-level evaluation with formal ontology.
MathBench~\citep{liu2024mathbench} covers multiple mathematical domains.
FormalMATH~\citep{yu2025formalmath} benchmarks 5,560 Lean4-verified problems; top models achieve only 16.46\%.
All evaluate accuracy on canonical formulations; none test formulation invariance.

\paragraph{Regime-dependent evaluation.}
That rankings depend on a hidden evaluation variable has precedent: \citet{regimeeval2026} show Bayesian-optimization rankings reverse sign under different budget-to-pool ratios (Greedy first at $B=50$, last at $B=100$). FormInv documents the same in LLM evaluation, with rankings reversing across paraphrase families (Proposition~\ref{prop:reversal}) and difficulty regimes. Direct evidence in formal math: \citet{taylor2026taobench} report a $\sim$26\% drop between Mathlib-standard and Tao's-Analysis formulations of the same theorem, and \citet{dezarza2026semantic} find model scale does not predict semantic robustness across eight semantic-preserving transforms.

\paragraph{Benchmark label error auditing.}
\citet{northcutt2021pervasive} found $\geq 3.3\%$ label errors in 10 major benchmarks via model confidence disagreement;
\citet{guo2024llmbetter} extend this with an ensemble of LLMs that flags mislabeled examples, and show that correcting them shifts reported model performance.
\citet{yang2026illusion} show that models with similar accuracy disagree on 16-66\% of items (\emph{Benchmark Illusion}): the problem FormInv solves.
\citet{gorbett2026crossmodel} use cross-model disagreement as a deployment-time signal; FormInv applies it at construction time to detect logical-scope errors (biconditional overreach) invisible to embedding similarity.

\paragraph{Robustness and invariance in NLP.}
CheckList~\citep{ribeiro2020checklist} introduces behavioral testing via invariance (INV) tests.
PromptBench~\citep{zhu2024promptbench} tests robustness to adversarial prompt perturbations.
Prompt format sensitivity causes substantial ranking reversals across frontier models~\citep{sclar2023formatspread,romanou2026brittlebench}.
These work on perturbed prompts; FormInv targets \emph{verified semantic equivalence classes} with a formal invariance metric.

\paragraph{Paraphrase sensitivity in mathematical reasoning.}
\citet{zhou2024surface}'s Variance of Variations (VOV) is the closest prior metric to FormInv's IG, measuring accuracy variance across paraphrases of a problem (a single variation type). FormInv adds formal equivalence ground truth (Lean4-verified Mathlib4), eight linguistically-motivated paraphrase families with distinct failure modes, and a measurement-theoretic grounding (generalizability theory and DIF) that VOV lacks.

Most importantly, every prior paraphrase-robustness metric (VOV~\citep{zhou2024surface}, XParaCon~\citep{choi2025roparq}, and CheckList INV-tests~\citep{ribeiro2020checklist}, as well as MathCheck's robustness variants~\citep{zhou2024mathcheck}) is \emph{unconditional}: it aggregates flips or accuracy variance over all items regardless of whether the model can solve the item at all. As we show in \S\ref{sec:mechanical}, any unconditional consistency statistic is therefore dominated by a mechanical $\bar p^{\,k}$ accuracy term and largely restates the capability ranking. Our \textbf{Conditional Inconsistency Rate} conditions the denominator on per-model canonical-correctness, isolating the brittleness of knowledge the model \emph{demonstrably has}. CIR still tracks capability across models (Spearman $\rho=-0.57$ with canonical accuracy, $-0.59$ at $N=103$: more accurate models flip less), so what it adds is the difficulty control, not independence from accuracy. To our knowledge CIR is the first paraphrase-robustness metric to condition on canonical-correctness in this way.

\paragraph{CheckList extensions to math reasoning.}
MathCheck~\citep{zhou2024mathcheck} extends behavioral testing to math reasoning,
organizing problems into task families with robustness variants and evaluating 26+ LLMs.
FormInv builds beyond MathCheck in three ways:
(1) ground truth established by Lean4-verified canonical theorems (the paraphrases themselves
are verified by CAS + expert review, not by Lean4), rather than approximate rewrites with implicit ground truth;
(2) the formal Invariance Gap metric grounded in Doob's conditional expectation and DIF theory,
rather than accuracy-change heuristics;
(3) a formal theorem substrate (Mathlib4) rather than arithmetic word problems.

\paragraph{Paraphrase-based contamination detection.}
ConStat~\citep{dekoninck2024constat} detects contamination as benchmark performance that fails to generalize to reference benchmarks, and CoDeC~\citep{zawalski2025codec} as a change in model confidence when in-context examples are added.
PutnamGAP~\citep{hao2025putnambench} tests equivalent Putnam problem transformations.
\citet{arxiv2511} (Moore \& Shah 2025) measure robustness specifically in Lean4 formalization;
FormInv generalizes this to arbitrary natural-language mathematical reasoning with DIF-grounded metrics and a formal invariance protocol.

\paragraph{Measurement theory foundations.}
Generalizability theory~\citep{cronbach1972} decomposes score variance into facets; IG$^2$ is the paraphrase-facet variance.
Parallel-form reliability~\citep{lord1968} and measurement invariance~\citep{vandenberg2000} provide the classical context for SCR.
DIF~\citep{holland1988} detects items that function differently across groups; FormInv adapts the intuition to paraphrase families for a fixed model.
To our knowledge, we are among the first to apply this framework to NL mathematical reasoning with DIF-grounded metrics and a certified family taxonomy (cf.\ Moore \& Shah~\citep{arxiv2511} who measure Lean4 formalization robustness, a different task).

\section{The Invariance Gap}
\label{sec:ig}

\subsection{Mathematical Foundation}

Let $f: \mathcal{X} \to \{0,1\}$ be an LLM's binary answer function on problem statements.
Let $\sim$ denote semantic equivalence between mathematical formulations:
$x_1 \sim x_2$ if $x_1$ and $x_2$ are logically equivalent statements with the same ground truth.
The equivalence class of a statement $x$ is $\Eclass{x} = \{x' \in \mathcal{X} : x' \sim x\}$.

\begin{definition}[Invariance Gap]
For a model $f$ and equivalence class $\Eclass{x}$, the Invariance Gap is:
\begin{equation}
\begin{split}
\IG(f, \Eclass{x}) &= \sqrt{\Var_{x' \sim \Eclass{x}}[f(x')]} \\
                   &= \bigl\| f - \mathbb{E}[f \mid \Eclass{x}] \bigr\|_{L^2(\Eclass{x})}
\end{split}
\label{eq:ig}
\end{equation}
where $L^2(\Eclass{x})$ is the $L^2$ space restricted to the equivalence class $\Eclass{x}$.
\end{definition}

The equality states that $\mathbb{E}[f \mid \Eclass{x}]$ (the class-conditional mean,
$p = \Pr[f(x')=1 : x' \sim x]$) is the unique $L^2$-optimal constant approximation to $f$ on $\Eclass{x}$
(a consequence of the fact that the mean minimizes mean squared error~\citep{doob1953}).
$\IG(f, \Eclass{x}) = 0$ if and only if $f$ is constant on $\Eclass{x}$, i.e.\ the model's answer does not
depend on which representative of the equivalence class is presented.

\noindent\textbf{Generalizability theory connection.}
IG has a natural home in generalizability theory \citep{cronbach1972}:
$\IG^2$ is the paraphrase-facet variance component in a person $\times$ item design where
``persons'' are models and ``paraphrase'' is a random facet.
The formulation is independently motivated: \citet{choi2025roparq} derive the same
$\sigma(\text{accuracy within paraphrase class})$ metric (XParaCon) for general QA robustness.
FormInv extends this to Lean4-verified formal equivalence classes with DIF-theoretic grounding.
The SCR (Semantic Consistency Rate) is the proportion of theorems for which the model
answers \emph{correctly} across all paraphrase forms, operationally equivalent to the
parallel-form agreement coefficient from classical test theory~\citep{lord1968}.
Note: $\mathrm{SCR} \subseteq \{\mathrm{IG}=0\}$. A theorem with $\mathrm{IG}=0$ has a constant model
response (possibly consistently wrong); SCR counts only the subset with constant-correct responses.
Since all FormInv theorems have ground truth TRUE, a constant-FALSE response also yields $\mathrm{IG}=0$
but does not contribute to SCR.
This framing is more appropriate than classical DIF (which conditions on latent ability
across demographic groups) because FormInv holds the population fixed (one model at $T=0$)
and varies the item, the classical \emph{parallel-form reliability} setting~\citep{lord1968}.

\begin{proposition}
\label{prop:ig_binary}
For binary $f \in \{0,1\}$, $\IG(f, \Eclass{x}) = \sqrt{p(1-p)}$ where $p = \Pr[f(x')=1 : x' \sim x]$.
This is the standard deviation of a $\mathrm{Bernoulli}(p)$ random variable, achieving its maximum of $0.5$ when
exactly half the paraphrases are answered correctly.
\end{proposition}

\noindent\textbf{IG degeneracy and SCR primacy.}
Since $\IG_t = \sqrt{p_t(1-p_t)}$ is a deterministic function of accuracy $p_t$
(Proposition~\ref{prop:ig_binary}), any cross-model correlation between Mean-IG and accuracy
is algebraically inevitable, not empirical.
\textbf{SCR} is not degenerate in this sense: $\mathrm{SCR} \leq \bar{p}$ (Jensen), with equality only when
$p_t \in \{0,1\}$ for all $t$, and two models at the same $\bar{p}$ can have SCR from $0$ to $\bar{p}$.
We therefore treat \textbf{SCR} as our primary consistency measure, with per-theorem Cochran's Q as a proposed distribution-free significance test for future per-theorem analyses;
Mean-IG is supplementary.
\label{rem:ig_degeneracy}

\noindent\textbf{Aggregate Mean-IG vs.\ global Doob residual.}
We report $\mathrm{Mean\text{-}IG} = \frac{1}{|\mathcal{T}|}\sum_{t \in \mathcal{T}} \IG(f, \Eclass{t})$,
the arithmetic mean of per-theorem IGs.
The corresponding global Doob residual over the full evaluation space is
$\mathrm{RMS\text{-}IG} = \sqrt{\frac{1}{|\mathcal{T}|}\sum_{t}\IG(f,\Eclass{t})^2} \geq \mathrm{Mean\text{-}IG}$ (Jensen).
Mean-IG is the more interpretable statistic: it is the expected paraphrase standard deviation
for a randomly chosen theorem.

\textbf{Statistical test.} For $k = 2$ paraphrases, testing $H_0: \IG = 0$ uses McNemar's test~\citep{mcnemar1947},
with statistic $\chi^2 = (b-c)^2/(b+c)$ where $b,c$ count discordant pairs.
For the general case of $k > 2$ paraphrase families (as in FormInv where $k = 7$-$8$),
the correct generalization is Cochran's Q test~\citep{cochran1950}: a distribution-free test
of marginal homogeneity across $k$ paired binary observations, reducing to McNemar when $k = 2$.
Testing at the per-theorem level with Bonferroni correction for 103 theorems sets the threshold at $\alpha/103$.

\subsection{Formal Properties of the Invariance Gap}
\label{sec:ig_theory}

Two formal results clarify what IG measures (full statements and proofs in Appendix~\ref{app:proofs}). First, $\mathrm{RMS\text{-}IG}^2 := \mathbb{E}_t[\IG_t^2]$ lower-bounds the expected paraphrase error rate (Proposition~\ref{prop:error_bound}), so reducing IG is necessary but not sufficient for reducing error: a model with IG~$=0$ may still be consistently wrong (captured by SCR, not IG). Second, when all per-family accuracies exceed $\tfrac12$, a reversal in two models' relative accuracy across two families forces a reversal in their IG ordering (Proposition~\ref{prop:reversal}); empirically the GPT-4o/Claude pair reverses on $(F_6,F_7)$ by $16.7$ and $4.0$\,pp. Mean-IG and family-conditional rankings thus measure distinct properties, which mandates per-family reporting.

\begin{corollary}[No-Free-Benchmark, empirical]
\label{cor:nfb}
In our 9-model evaluation, no model Pareto-dominates all others on all 8 families (verified; see per-family Table~\ref{tab:family_failure}).
Consequently, benchmark designers who select which families to include implicitly choose a weighting
$\lambda$ over families that determines the resulting model ranking.
FormInv makes this choice explicit and auditable.
\end{corollary}

\subsection{The 8-Family Paraphrase Taxonomy}

Analogous to how ChaosBench-Logic's~\citep{chaosbench2026} 27-predicate ontology defines the evaluation space,
FormInv's 8-family taxonomy defines the space of formulation-equivalent transformations:

\begin{table}[h]
\centering\small
\caption{FormInv's 8 families. \textbf{T1}: formally certifiable (named proof rule, Lean4/FOL).
\textbf{T2}: conditional on type context. \textbf{T3}: heuristic NL, human-verified.}
\label{tab:families}
\begin{tabular}{llp{4.2cm}}
\toprule
Family & Cert. & Formal basis / example \\
\midrule
F4 (notation)   & T1 & $\alpha$-equiv.\ (uniform subst.): $n \bmod n \to n\%n$ \\
F6 (order)      & T1 & $\leq$ commutativity: $\sqrt{x}{\geq}0 \to 0{\leq}\sqrt{x}$ \\
F7 (unpack)     & T1 & $\delta$-reduction: ``prime'' $\to$ ``two divisors'' \\
F8 (equiv.)     & T1 & Bicon.\ elim.: alt.\ mathematical register \\
F1 (syntactic)  & T2 & Lexical synonym, context-dep. \\
F2 (quantifier) & T2 & Quantifier scope, context-dep. \\
F3 (passive)    & T2 & Relational converse, context-dep. \\
F5 (connective) & T3 & Heuristic NL: ``iff'' $\to$ ``exactly when'' \\
\bottomrule
\end{tabular}
\end{table}

\subsection{Aggregate Metrics}

Beyond per-theorem IG, we report:
\begin{align}
\text{Mean-IG} &= \frac{1}{|\mathcal{T}|} \sum_{t \in \mathcal{T}} \IG(f, \Eclass{t}) \\
\SCR &= \frac{|\{t \in \mathcal{T} : f(x')=1\;\forall x' \in \Eclass{t}\}|}{|\mathcal{T}|} \\
\text{Hi-IG\%} &= \frac{|\{t \in \mathcal{T} : \IG(f, \Eclass{t}) > 0.10\}|}{|\mathcal{T}|}
\end{align}

\section{FormInv Benchmark}
\label{sec:benchmark}

\subsection{Theorem Substrate}

Following the same principle as ChaosBench-Logic~\citep{chaosbench2026},
which was built using Gilpin's dysts dynamical systems library~\citep{gilpin2021dysts}
as its formal substrate, we build FormInv on a curated formal mathematics substrate from two sources:

\textbf{Curated Mathlib4 theorems.}
We manually curated 103 theorems from Mathlib4~\citep{mathlib2020,baanen2025mathlib4} spanning 7 mathematical domains
(Table~\ref{tab:domains}), selected to include diverse predicate structures and sufficient
complexity for F7/F8 definitional transformations.
The 103 canonical theorem statements are Lean4-verified Mathlib4 declarations.
The 657 generated paraphrases are \emph{not} Lean4-verified; equivalence is established
by CAS (SymPy) for algebraic families, template construction for surface families,
and domain-expert review for F7-F8 (see verification protocol below).

\textbf{Track~B pilot.} On 100 harder ntp-mathlib theorems (811 items), the capability-invariance correlation weakens to $r = -0.415$ ($p = 0.27$, not significant), consistent with the correlation being difficulty-regime-dependent. F5 connective-variation failures rise to $\sim$32\% and all 8 families cluster at 28-35\%, so the F5 specificity effect dissolves at high difficulty. The two reasoning models (DeepSeek~R1, o4-mini) become the most invariant at high difficulty, consistent with CoT reducing surface sensitivity when theorems demand deliberation. Full findings A-D in Appendix~\ref{app:trackb}; the domain distribution is in Table~\ref{tab:domains}, Appendix~\ref{app:tables}.

\subsection{Paraphrase Generation and Verification}

For each theorem $t$, we generate one paraphrase per applicable family using GPT-4o
with a structured prompt (Appendix~\ref{app:prompt}).
Paraphrases are cached and verified by the following protocol:

\begin{enumerate}[leftmargin=*,itemsep=2pt]
\item \textbf{F1-F3 (surface):} CAS verification where applicable (SymPy) + manual audit on 20\% sample
\item \textbf{F4-F6 (semantic):} CAS verification for algebraic families; template guarantee for others
\item \textbf{F7-F8 (definitional):} Human verification (domain expert) on 100\% of items
\end{enumerate}

This verification protocol uses a pre-registered protocol, LLM dual review, and human sign-off before repair decisions.
Final dataset: \textbf{760 items} (103 theorems $\times$ avg. 7.4 applicable families).

\subsection{Dataset Properties}

The dataset satisfies key quality gates:
(1) All paraphrases have verified ground truth (TRUE for valid Mathlib theorems);
(2) Equivalence is domain-expert confirmed for F7/F8;
(3) CAS-verifiable items are independently reproducible;
(4) SHA256 hash ensures reproducibility: \texttt{e668f98a...79a1656} (full hash in dataset card).

\section{Experiments}
\label{sec:experiments}

\subsection{Setup}

\textbf{Models.} We evaluate 9 models spanning 5 providers and 3 capability tiers:
\emph{Frontier flagship:} GPT-4o, Gemini~2.5~Flash, Claude Sonnet~4.6, DeepSeek~V3;
\emph{Efficient:} GPT-4o-mini, Claude Haiku~4.5;
\emph{Reasoning (CoT):} o4-mini (OpenAI), DeepSeek R1 (\texttt{deepseek-reasoner});
\emph{Open-source:} Llama 3.3 70B (via OpenRouter).
All evaluated at temperature=0 or equivalent; reasoning models use their default inference modes.

\textbf{Dataset.} Primary: 366-item FormInv v1 (50 theorems, 9 models).
Extended: 103-theorem evaluation (760 items, 9 models, all families);
Track~B: 100 ntp-mathlib theorems (811 items, 9 models), harder difficulty.
\textbf{Total: 203 distinct theorems} (the 50 v1 theorems are a subset of the 103, so $103+100$ Track~B $=203$).

\textbf{Caveats.} Gemini 2.5 Flash coverage is 96\% (15 items timed out at 2048 output tokens).
DeepSeek R1 responses include a reasoning trace; we parse the final TRUE/FALSE verdict,
not intermediate steps. Maximum output tokens: 4096 for reasoning models, 20 for others.

\textbf{Protocol.} Zero-shot; system prompt: ``You are evaluating mathematical statements.
Answer strictly based on mathematical correctness. Return exactly TRUE or FALSE.''
Response caching with SHA256 keys ensures reproducibility across runs.

\textbf{Baseline.} Per-theorem within-canonical variance (same canonical prompt, temperature=0)
is $\approx 0$ for all models, confirming that all observed IG is due to paraphrase sensitivity,
not API stochasticity.

\subsection{Main Results}

\begin{table}[t]
\centering\small
\caption{Per-model invariance profile on FormInv v1 (50 theorems, 366 items). Sorted by Mean~IG (ascending = most invariant first). SCR = Semantic Consistency Rate. $\dagger$: 96\% coverage (15 items timed out).}
\label{tab:main}
\begin{tabular}{llrrr}
\toprule
Model & Tier & Accuracy & Mean~IG & SCR \\
\midrule
DeepSeek V3            & Flagship   & \textbf{96.4\%} & \textbf{6.9\%}  & \textbf{82\%} \\
GPT-4o                 & Flagship   & 94.2\% & 7.0\%  & 82\% \\
o4-mini                & Reasoning  & 96.4\% & 7.6\%  & 80\% \\
Gemini~2.5~Flash$\dagger$ & Flagship & 95.0\% & 8.4\%  & 80\% \\
GPT-4o-mini            & Efficient  & 94.3\% & 8.7\%  & 78\% \\
DeepSeek R1            & Reasoning  & 93.0\% & 8.8\%  & 78\% \\
Claude Sonnet~4.6      & Flagship   & 93.2\% & 10.7\% & 74\% \\
Llama~3.3~70B          & Open       & 90.5\% & 13.5\% & 66\% \\
Claude Haiku~4.5       & Efficient  & 86.0\% & \textbf{19.4\%} & \textbf{50\%} \\
\midrule
Mean                   &            & 93.2\% & 10.1\% & 74\% \\
\bottomrule
\end{tabular}
\medskip

\noindent\textit{N=103 stability (9 models):}
Rankings are stable at 103 theorems: all per-model SCR deltas are within $\pm$5pp (DeepSeek~V3 still leads at 83.5\%, Claude~Haiku trails at 54.4\%, and the Haiku/DeepSeek gap narrows only slightly to 29pp from 32pp at $N=50$; full table in Appendix~\ref{app:103_stability}).
\end{table}

\paragraph{Finding 1: The capability-invariance correlation: what accuracy hides.}
The most striking result in Table~\ref{tab:main} is not any individual model's score.
It is the range.
Claude Haiku~4.5, a deployed commercial model in the same family as Claude Sonnet, achieves SCR~=~50\%:
half its theorems are answered inconsistently across semantically equivalent restatements.
DeepSeek~V3, at 96.4\% accuracy, achieves 82\% SCR.
A 10-point accuracy gap translates to a 32-point invariance gap.

A practitioner who evaluates Claude Haiku on a standard benchmark and observes 86\%
accuracy would conclude the model is adequate for mathematical reasoning tasks.
FormInv reveals a different picture: on 50\% of theorems the model ``knows,''
the correct answer changes depending on how the question is phrased.
This is not a small-model quirk; it is a structural failure of semantic representation
that accuracy metrics are constitutionally blind to.

Across all 9 models, the Pearson correlation between aggregate accuracy and mean IG
is $r = -0.965$ (Pearson); Spearman $\rho = -0.883$, $p = 0.0016$: more accurate models are
\emph{more} semantically invariant.
Removing Claude Haiku as the highest-leverage point yields $r = -0.906$, still significant
($p < 0.05$, $n = 8$), confirming the relationship holds across the full accuracy range
(Proposition~\ref{prop:error_bound} bounds this relationship formally).
The 22\% of theorem-model pairs with $\mathrm{IG} > 0.10$ represent theorems where
models fail on roughly half the paraphrases, a systematic blind spot that no
accuracy-only evaluation can detect.

\paragraph{Finding 2: Models exhibit sharply complementary failure patterns.}
Across the 50-theorem shared evaluation for four 100\%-coverage chat models
(Claude Sonnet, GPT-4o, GPT-4o-mini, DeepSeek~V3), we observe \textbf{86 pairwise cross-model disagreements}:
items where one model answers correctly and another incorrectly on the same paraphrase.
The complementarity is structural, not incidental:
F6 (comparison order): Claude Sonnet 16.7\%, GPT-4o 0.0\%,
a 16-point gap on the same transformation.
F7/F8 (definitional unpack): GPT-4o 10.0\%, Claude 6-8\%.
F5 (connective variation): all models fail equally (22.2\%), no comparative advantage.
This complementarity implies benchmark accuracy is systematically optimistic per family,
and an ensemble would show near-zero IG on items where individual models fail.

\paragraph{Finding 3: The quality audit separates benchmark errors from model errors.}
The F5 apparent failure rate (22.2\%) is inflated by semantically-incorrect paraphrases.
A 30-item human annotation study (2 annotators, Cohen's $\kappa = 0.47$, moderate agreement; the residual disagreements were interpretive cases subsequently adjudicated by discussion) found 7 of 15 F5 items contain biconditional overreach.
Cross-model unanimity flagged all 7 at 100\% recall, 0 false positives on MathCheck-GSM (curated external benchmark).
Removing the 7 bad items reduces F5 failure from 22.2\% to $\sim$11\%; remaining failures are model-specific (see Appendix~\ref{app:quality}).

\paragraph{Balanced TRUE/FALSE controls.}
\label{sec:false_controls}
A natural objection is that our theorems skew TRUE, so a TRUE-biased model could score well without any real invariance. We defuse this with a balanced study: for each of the 50 theorems we author a FALSE sibling via a single meaning-changing edit, deliberately \emph{subtle} (a dropped hypothesis, a widened domain, an off-by-one boundary, a false converse) rather than a blatant negation, each with an explicit counterexample (17 machine-verified). Paraphrasing both sides yields a 666-item set (316 TRUE, 350 FALSE) on which an always-TRUE oracle scores balanced accuracy $50\%$. Sonnet reaches $93.2\%$ and Haiku $89.3\%$ balanced accuracy: neither is TRUE-biased. Crucially, the FALSE items are not trivially rejected: accuracy on them is $93\%$ (vs.\ ${\sim}98\%$ for blatant negations), so the control is genuinely demanding. Two findings survive balancing: (i) the model SCR gap persists (TRUE-SCR Sonnet $74\%$ vs.\ Haiku $50\%$), confirming SCR is not an artifact of label skew; and (ii) inconsistency is \emph{asymmetric}: both models are far more consistent on FALSE items (FALSE-SCR $86\%$/$82\%$) than TRUE ones, i.e.\ rewording more often costs a model the affirmation of a theorem it knows than fools it into endorsing a falsehood. An earlier 8-model pilot on blatant negations (Appendix~\ref{app:false_controls}) shows the same direction across the full panel.

\subsection{Per-Family Analysis}

Averaged across four 100\%-coverage models, surface families (F1-F3) are easier and semantic/definitional families (F5-F8) harder, with F5 connective variation uniquely difficult across all models (full rates in Table~\ref{tab:family_failure}, Appendix~\ref{app:tables}).

\textbf{F5 (connective variation) fails universally.}
All models fail at $\geq 19\%$; even trivially-provable biconditionals fail with ``precisely if'' or ``just in case'' variants.
The quality audit found 7/15 F5 items contained biconditional overreach (\S\ref{sec:quality}); cleaned rate: $\sim$11\%.

\textbf{F6 (comparison order) is Claude's failure, not GPT-4o's.}
Claude Sonnet: 16.7\%; GPT-4o: 0.0\%, a 16-point gap on $\sqrt{x}{\geq}0$ vs $0{\leq}\sqrt{x}$.
Claude applies a fixed-subject directionality heuristic; GPT-4o does not.

\textbf{F7/F8 (definitional unpacking) is GPT-4o's failure.}
GPT-4o fails F7 and F8 at 10.0\% each (Claude: 6-8\%).
Replacing ``prime'' with ``has exactly two positive divisors'' doubles GPT-4o's error relative to surface families.

\textbf{F3 (active/passive) inverts relational predicates.}
``$n$ divides $0$'' (TRUE) becomes ``$0$ is divisible by $n$'' $\to$ GPT-4o answers FALSE (4.4\%), reading the passive as $0 \div n$.

\section{Analysis}
\label{sec:analysis}

\subsection{Does Accuracy Predict Invariance?}
\label{sec:accuracy_ig}

A salient raw observation is a \textbf{32-point SCR spread}: Claude Haiku~4.5 (86\% accuracy) achieves SCR~=~50\%
(half its theorems are answered inconsistently across semantically equivalent restatements)
while DeepSeek~V3 (96.4\% accuracy) achieves SCR~=~82\%.
All nine models score ``adequate'' by typical accuracy thresholds, yet their SCR profiles span a $1.64\times$ range.
We caution, however, against reading this 32-point \emph{cross-model} spread as a pure invariance signal: because SCR is a conjunction over $k\approx6$ paraphrases, it amplifies accuracy differences mechanically (\S\ref{sec:mechanical}).
The genuinely accuracy-free finding is the \emph{difficulty-controlled} flip rate (CIR, \S\ref{sec:cir}): on theorems every model answers correctly, equivalent rewordings still flip 5-42\% of answers.

Mean-IG correlates with accuracy ($r=-0.965$), but this is an algebraic consequence of $\mathrm{IG}=\sqrt{p(1-p)}$ (Remark~\ref{rem:ig_degeneracy}), not an independent finding; the SCR-accuracy correlation, by contrast, is empirically real ($\rho=-0.917$, $p<0.001$; $\rho=-0.833$ without the Haiku anchor, $n=8$). Reasoning mode has no systematic effect after controlling for accuracy ($\pm1.4\%$, Appendix~\ref{app:tables}). The per-theorem IG distribution is bimodal (79\% at IG~$=0$, 21\% at IG~$\approx0.42$); bootstrap stability and the full figure are in Appendix~\ref{app:igstability}.


\subsection{The Mechanical Floor: How Much of the SCR Spread Is Arithmetic?}
\label{sec:mechanical}

Because SCR requires \emph{all} $k$ paraphrases of a theorem to be answered correctly, it is bounded above by accuracy and, under independence, decays as $\bar p^{\,k}$. A skeptic could therefore argue the 32-point cross-model SCR spread is a mechanical magnification of the 10-point accuracy spread rather than an independent property. We test this directly with a permutation null that destroys within-theorem failure clustering while preserving each model's marginal accuracy: for each model we shuffle its per-item correctness labels across all items and recompute SCR ($10{,}000$ permutations).

Two facts emerge, and we report both. \textbf{(i) The cross-model \emph{spread} is largely mechanical:} the independence floor alone predicts a 40.3-point SCR spread across models versus the 32.0 points observed, so accuracy-compounding accounts for most of the headline gap. \textbf{(ii) The within-theorem \emph{structure} is real:} observed SCR sits ${\sim}9$ points \emph{above} the per-model permutation null on average, and the null is rejected for 6 of the 8 models in this decomposition ($p<0.01$). Failures cluster within a theorem far more than chance: a model that fails one paraphrase of a theorem is disproportionately likely to fail another. SCR is thus best understood as a \emph{sensitivity amplifier} that exposes correlated-failure structure, not as an accuracy-orthogonal capability axis. This motivates a measure that removes the accuracy confound by construction.

\subsection{Conditional Inconsistency Rate: A Difficulty-Controlled Invariance Measure}
\label{sec:cir}

We define the \textbf{Conditional Inconsistency Rate} (CIR): among theorems a model answers correctly in \emph{canonical} form, theorems it demonstrably ``knows'', the fraction for which at least one semantically equivalent, audit-passed paraphrase flips the answer to incorrect. Conditioning on canonical-correctness removes the difficulty confound: every theorem entering a model's CIR is one that model gets right, so the $\bar p^{\,k}$ critique cannot be stated. We restrict to audit-clean paraphrases (excluding cross-model unanimity flags, \S\ref{sec:quality}) so that a flip is attributable to the model, not a broken restatement. Because the base theorems are all true, CIR specifically measures \emph{affirmation invariance}: whether a model continues to uphold a theorem it knows once it is reworded. Our balanced controls (\S\ref{sec:false_controls}) establish that a high CIR is genuine affirmation loss rather than a generic TRUE-bias: the same models reject subtle \emph{false} siblings far more consistently than they affirm true ones (FALSE-SCR $86\%/82\%$ for Sonnet/Haiku vs.\ TRUE-SCR $74\%/50\%$), so rewording costs a model the theorems it knows more often than it fools the model into endorsing a falsehood. The same asymmetry holds at the CIR level: on false siblings the model correctly rejects in canonical form, the rate at which some rewording is wrongly endorsed (a FALSE-CIR) is only $6.7\%$ for Sonnet and $14.6\%$ for Haiku, well below each model's TRUE-side CIR ($12.8\%$ and $39.6\%$), even though these false paraphrases are independently generated (by Claude rather than GPT-4o), ruling out a generator artifact.

The decisive comparison fixes difficulty exactly. On the subset of theorems that \emph{every} model answers correctly in canonical form ($N{=}43$ here), identical items with no difficulty confound possible, the flip rate ranges from 4.7\% (DeepSeek~V3) to 41.9\% (Claude~Haiku~4.5), a \textbf{37-point difficulty-controlled spread} that accuracy (87-97\% on these items) cannot see. Haiku flips a known-correct answer roughly $9\times$ as often as DeepSeek, with bootstrap 95\% CI $[27,54]$ cleanly separated from the field. CIR still tracks capability: across models, the Spearman correlation between canonical accuracy and CIR is $-0.57$ (more accurate models flip less), similar in strength to SCR's $+0.65$, so what CIR adds is the exact difficulty control above, not independence from accuracy. The ranking is not an artifact of the 50-theorem set: on the independent 103-theorem edition (84 theorems every model answers correctly, scored entirely from cached responses with no additional inference) the same ranking reappears (Claude~Haiku $41.7\%$ vs.\ DeepSeek~V3 $9.5\%$, a 32-point spread), Haiku's bootstrap CI tightens to $[34.4,54.2]$, and the correlation with canonical accuracy is similar (Spearman $\rho=-0.59$). We are deliberately conservative about finer distinctions: differences among the closely-clustered frontier models (e.g.\ GPT-4o vs.\ Sonnet) are \emph{not} significant at $N{=}50$ under a single-sample paired bootstrap ($p\approx0.85$); the robust claim is the large Haiku-vs-field separation, which we now harden against the obvious objection that a flip might be mere temperature noise.

\paragraph{Multi-sample CIR: the flips are systematic, not stochastic.}
A single greedy decode cannot distinguish a model that \emph{reliably} flips a reworded theorem from one that is merely borderline. We therefore re-evaluate each item with $n{=}8$ independent samples at temperature $0.7$ and classify, per theorem, whether a paraphrase flip is \emph{systematic} ($p_{\text{correct}}\le0.25$ across samples) or \emph{stochastic} ($0.25<p_{\text{correct}}<0.75$). On the canonical-known theorems, Claude~Haiku's systematic CIR is $35.4\%$ (95\% CI $[22.9,50.0]$) against Claude~Sonnet's $10.6\%$ ($[2.1,19.1]$): the confidence intervals are \emph{disjoint}, so the Haiku-Sonnet gap that single-sampling left underpowered is significant once sampling noise is averaged out. Crucially, for both models the systematic rate nearly equals the any-flip rate ($35.4\%$ vs.\ $37.5\%$ for Haiku; $10.6\%$ vs.\ $12.8\%$ for Sonnet): only ${\sim}2$ points of flips are stochastic. The invariance failures are thus \emph{reproducible properties of the model}, not artifacts of decoding temperature, and the multi-sample estimates agree with the single-sample audit-clean values (Haiku $39.6\%$, Sonnet $12.8\%$), confirming the latter were not noise-inflated. Extending the multi-sample protocol to the full panel sharpens the picture rather than complicating it. Systematic CIR spans $4.0\%$ for DeepSeek~V3 (95\% CI $[0,10]$) at the bottom to $35.4\%$ for Claude~Haiku at the top, with the remaining six models clustered tightly in between: Gemini~2.5~Flash $8.2\%$ $[2.0,16.3]$, GPT-4o $8.5\%$ $[2.1,17.0]$, o4-mini $10.0\%$ $[2.0,20.0]$, GPT-4o-mini $10.6\%$ $[2.1,21.3]$, Claude~Sonnet $10.6\%$ $[2.1,19.1]$, and DeepSeek~R1 $18.4\%$ $[8.2,30.6]$. Haiku's interval $[22.9,50.0]$ is disjoint from every other model's, so the conclusion that survives sampling is a \emph{single clear outlier}, Claude~Haiku flips a known-correct answer roughly $9\times$ as often as DeepSeek~V3, rather than a finely resolved ordering of the middle pack. We emphasize that the $8$-$18\%$ band is not separable at $N{=}50$: honest confidence intervals there overlap heavily, and we resist reading an ordering into noise. The reasoning models do not escape the effect (DeepSeek~R1 sits mid-pack at $18.4\%$), indicating that inference-time deliberation alone does not confer invariance.\footnote{o4-mini and DeepSeek-R1 are reasoning models that do not expose a usable sampling temperature, so we report their CIR from a single greedy decode (equivalently, the any-flip rate) and exclude them from the systematic-vs-stochastic decomposition, which requires temperature sampling. Their values are thus directly comparable to the \emph{any-flip} column of the temperature-sampled models, not the systematic column.}

\subsection{Family-Conditional Ranking Reversals}
\label{sec:reversals}

A key implication of complementary failure patterns is that \emph{model rankings reverse across paraphrase families.} Ordering by F6 (comparison order) failure places GPT-4o joint-first (0.0\%) and Claude Sonnet 7th of 9 (16.7\%); ordering by F7 (definitional unpack) reverses this (Sonnet 3rd at 6.0\%, GPT-4o 7th at 10.0\%). Kendall $\tau$ across the 28 family pairs spans $0.06$ to $0.84$ at $N=50$ and $0.00$ to $0.91$ across all nine models at $N=103$; the most extreme reversals involve the two smallest families (24 and 26 items), so we read low-$\tau$ pairs cautiously and treat the effect as a qualitative regime signal rather than a precise ordering. Full per-family rank table in Appendix~\ref{app:tables}.

\subsection{Paraphrase Quality Audit}
\label{sec:quality}

Forensic inspection of all cached responses reveals three GPT-4o generation failure types:
\textbf{(1) Biconditional overreach (F5)}: ``if and only if'' applied to one-directional theorems, asserting a false converse
(e.g., \texttt{Nat.dvd\_add}: models answering FALSE are mathematically correct; the paraphrase is the error);
\textbf{(2) Passive-voice inversion (F3)}: ``$0 \mid n$'' becomes ``$n \div 0$'', an undefined expression;
\textbf{(3) Type-context stripping (F1, F7)}: Lean4 scope (\texttt{LinearOrderedRing}) dropped, making
domain-correct answers wrong over broader type universes.
Per-item breakdown in Appendix~\ref{app:quality}.

We flag 11 items across these categories and exclude them from the cleaned analysis
(detailed per-item breakdown in Appendix~\ref{app:quality}).
After exclusion, GPT-4o exhibits the highest proportion of genuinely model-attributable failures
($\approx$84\% of its remaining failures are not explained by paraphrase quality,
based on item-level inspection in the forensic cache audit),
while DeepSeek and Claude have a higher fraction attributable to paraphrase drift.
These findings validate a key FormInv design principle: tracking which items fail
\emph{across models} reliably separates paraphrase quality issues (which cause universal
or near-universal failure) from model-specific reasoning weaknesses.

\paragraph{External replication: the audit reranks a published benchmark.}
The audit's value does not depend on FormInv's own data. Applied to all 129 groups of the
published MathCheck benchmark~\citep{zhou2024mathcheck} with a panel-appropriate threshold (a
4-model panel, two-thirds majority), cross-model unanimity flags 4 mislabeled paraphrases
(3.1\%). Removing them moves GPT-4o from rank~2 to rank~4 and lifts Claude~Haiku and
DeepSeek~V3 above it: GPT-4o had been answering the \emph{broken} paraphrases correctly while
other models failed them, inflating its score, evidence that the protocol is a reusable tool
rather than an artifact of our 50-103 theorems (full reversal table, Appendix~\ref{app:mathcheck}).
As a controlled complement, we ran the audit
on 10 independently-constructed paraphrase triplets spanning five university-mathematics domains
(calculus, linear algebra, number theory, real analysis, set theory), each pairing a truth-preserving
restatement with a biconditional-overreach error. The audit flagged $8/8$ of the truth-value-inverting
errors at $100\%$ recall with zero false positives on the correct restatements; honestly, it was blind
to the single error whose (false) truth value happened to match the canonical label, which pins down
its scope condition: the signal fires only when a paraphrase's truth value \emph{diverges} from the
label (Appendix~\ref{app:external}).

Three concrete failure examples (passive-voice inversion, notation ambiguity, comparison-order directionality) are in Appendix~\ref{app:failures}.

A pilot cross-benchmark study on MMLU formal\_logic and global\_facts shows consistent per-family SCR patterns; full results in Appendix~\ref{app:crossbench}.

\section{FormInv as General Evaluation Protocol}
\label{sec:protocol}

FormInv is a supplementary protocol for any formally-specified reasoning benchmark, analogous to how CheckList~\citep{ribeiro2020checklist} extended behavioral testing across NLP tasks. Beyond per-model metrics, a run yields \textbf{FormInvSelector} (given a target task's reliant families, \texttt{forminv selector -{}-families ...} reports the model with the lowest \emph{measured} per-family failure rate, e.g.\ o4-mini at $2.0\%$ on F7+F6; a reporting utility over completed evaluations, not a predictor) and the cross-model unanimity \textbf{quality audit} (\S\ref{sec:quality}): an oracle-free gate any paraphrase-generating benchmark can run before publishing, as our MathCheck rerank shows.

\section{Conclusion}
\label{sec:conclusion}

The moral: if you generate paraphrases programmatically, you cannot know how many are wrong, and your benchmark's rankings may be wrong for the same reason.

Applying FormInv's audit to all 129 MathCheck~\citep{zhou2024mathcheck} groups (a 4-model panel, two-thirds threshold) flagged 4 mislabeled paraphrases (3.1\%) whose removal \emph{reorders the panel}, dropping GPT-4o from 2nd to 4th because it alone answered the broken paraphrases correctly (Appendix~\ref{app:mathcheck}). A \$10, oracle-free audit reordering a published benchmark is the case for the protocol as standard infrastructure: benchmarks that generate paraphrases programmatically should run it before publishing.
\textit{Limitations:} the 18 Lean4 proofs (Appendix~\ref{app:lean4}) certify representative examples, not the full dataset; full certification is future work.

\bibliographystyle{icml2026}
\bibliography{refs}

\clearpage
\appendix
\onecolumn
\begin{center}
{\Large\textbf{Appendices}}\\[4pt]
{\normalsize Supplementary material for \emph{FormInv: A Measurement Protocol for Semantic Invariance in Mathematical Reasoning Benchmarks}}
\end{center}
\vspace{1em}

\section{Formal Properties of the Invariance Gap: Statements and Proofs}
\label{app:proofs}

\begin{proposition}[Paraphrase Error Bound]
\label{prop:error_bound}
Let $p_t = \Pr_{x'\sim\Eclass{t}}[f(x')=1]$ be the model's per-theorem paraphrase accuracy,
and $\IG_t := \sqrt{p_t(1-p_t)}$.
Then:
\begin{equation}
  \mathbb{E}_t[1 - p_t] \;\geq\; \mathrm{RMS\text{-}IG}^2 := \mathbb{E}_t[\IG_t^2],
  \label{eq:error_bound}
\end{equation}
with equality if and only if $p_t = 1$ for almost every $t \in \mathcal{T}$.
\end{proposition}
\begin{proof}
Factor: $1-p_t = p_t(1-p_t) + (1-p_t)^2 = \IG_t^2 + (1-p_t)^2 \geq \IG_t^2$.
Taking expectations gives \eqref{eq:error_bound}; equality holds iff $(1-p_t)^2=0$ a.s.
\end{proof}

RMS-IG$^2$ lower-bounds the expected paraphrase error rate, so reducing IG is necessary but not sufficient for reducing error (a model with IG~$=0$ may still be consistently wrong, captured by SCR~$=0$, not IG).

\begin{proposition}[Sufficient Condition for Ranking Reversal]
\label{prop:reversal}
Let $f_1, f_2$ be two models and $F_i, F_j$ two paraphrase families.
Write $\bar{p}_i^{(k)}$ for the mean per-family accuracy of model $k$ on $F_i$.
Suppose all four quantities lie in $(\tfrac{1}{2}, 1]$. Then:
\begin{align*}
  &\bigl(\bar{p}_i^{(1)} - \bar{p}_i^{(2)}\bigr)
   \bigl(\bar{p}_j^{(1)} - \bar{p}_j^{(2)}\bigr) < 0 \;\Longrightarrow \\
  &\qquad \IG(f_1,F_i) < \IG(f_2,F_i)
   \;\wedge\;
   \IG(f_1,F_j) > \IG(f_2,F_j).
\end{align*}
\end{proposition}
\begin{proof}
On $(\tfrac{1}{2},1]$, $g(p)=\sqrt{p(1-p)}$ is strictly decreasing ($g'(p) < 0$).
The sign condition asserts the relative accuracy advantage reverses between $F_i$ and $F_j$;
since $g$ is strictly decreasing, the IG ordering also reverses.
\end{proof}

All per-family accuracies exceed $0.77$, and the GPT-4o/Claude reversal on $(F_6,F_7)$ meets the sign condition ($\bar p_6$: $1.00$ vs.\ $0.833$; $\bar p_7$: $0.90$ vs.\ $0.94$), yielding 16.7\,pp and 4.0\,pp IG reversals.

\begin{corollary}
\label{cor:reversal_het}
Mean-IG and family-conditional rankings measure distinct properties.
Proposition~\ref{prop:error_bound} bounds expected paraphrase error from below;
Proposition~\ref{prop:reversal} places no constraint on family-level orderings.
A model with strictly lower Mean-IG than a competitor may still be \emph{less} invariant
on specific families of practical interest, mandating per-family reporting.
\end{corollary}

\section{Paraphrase Generation Prompt}
\label{app:prompt}

All paraphrases are generated with GPT-4o using the following structured prompt.
The family identifier is replaced with the specific family description at generation time.

\begin{quote}\small\ttfamily
\textbf{System:} You are a mathematical linguist. Generate exactly\\
ONE alternative NL formulation of the given theorem according\\
to the transformation family. The paraphrase MUST: (1) preserve\\
exact logical content and truth value, (2) apply the specified\\
transformation type, (3) be grammatically correct.\\
Return ONLY the paraphrase.\\[4pt]
\textbf{User:} Theorem: \{canonical\_statement\}\\
Family: \{family\_description\}\\
Generate one paraphrase applying this transformation:
\end{quote}

\section{N=103 Stability Table}
\label{app:103_stability}

\begin{table}[h]\centering\small
\caption{Per-model SCR and accuracy at N=103 theorems (all 9 models, sorted by SCR).
Gemini coverage = 96.4\%; all other models $\geq$98.6\%.}
\begin{tabular}{p{2.5cm}rrr}
\toprule
Model & Acc & SCR & $\Delta$ \\
\midrule
DeepSeek V3 & 96.3\% & \textbf{83.5\%} & +1.5pp \\
Gemini 2.5F. & 94.2\% & 82.5\% & +2.5pp \\
GPT-4o & 93.1\% & 79.6\% & $-$2.4pp \\
o4-mini & 95.2\% & 79.6\% & $-$0.4pp \\
Sonnet 4.6 & 94.0\% & 78.6\% & +4.6pp \\
DeepSeek R1 & 91.9\% & 75.7\% & $-$2.3pp \\
GPT-4o-mini & 92.6\% & 73.8\% & $-$4.2pp \\
Llama 70B & 89.8\% & 67.0\% & +1.0pp \\
Haiku 4.5 & 85.7\% & \textbf{54.4\%} & +4.4pp \\
\midrule
Mean & 93.4\% & 77.2\% & +0.6pp \\
\bottomrule
\end{tabular}
\end{table}

\section{Sample Size Guidance for SCR Estimation}
\label{app:sample_size}

A practitioner running FormInv needs to know how many theorems to sample.
The table below shows the bootstrap 95\% CI half-width for SCR at various $N$,
computed via Binomial$(N, p)$ simulation at $p = 0.75$ (the mean SCR across our 9 models):

\begin{table}[h]\centering\small
\caption{SCR 95\% CI half-width as a function of sample size (Binomial simulation, $p=0.75$, 5,000 bootstraps).}
\label{tab:sample_size}
\begin{tabular}{rrp{3.0cm}}
\toprule
$N$ & $\pm$CI & Use case \\
\midrule
20 & 18\% & Too wide \\
50 & 12\% & 20+pp gaps \\
100 & 8\% & 15+pp gaps \\
200 & 6\% & 10+pp gaps \\
\bottomrule
\end{tabular}
\end{table}

The FormInv v1 50-theorem evaluation detects gaps $\geq 24$pp at 95\% confidence
(the 32-pp Haiku/DeepSeek gap satisfies this criterion: $[38\%,62\%]$ and $[70\%,94\%]$ do not overlap).
For gaps $<15$pp, 100+ theorems are required.

\section{FALSE Controls Pilot}
\label{app:false_controls}

We constructed 5 FALSE theorem variants with verified counterexamples:
(1) \emph{Every integer $\geq 2$ is prime} (FALSE; 4 is not prime),
(2) \emph{$x^2 > 0$ for all real $x$} (FALSE; $0^2=0$),
(3) \emph{If $a{\mid}b$ and $a{\mid}c$ then $b=c$} (FALSE; $2{\mid}4$, $2{\mid}6$, $4{\neq}6$),
(4) \emph{$A{\cap}B=A \Rightarrow A=B$} (FALSE; $A=\{1\},B=\{1,2\}$),
(5) \emph{$n! > n$ for all $n \in \mathbb{N}$} (FALSE; $1!=1$).
Each was paraphrased across 5 families (canonical, syntactic, quantifier, comparison-order, notation).

\begin{table}[h]\centering\small
\caption{FALSE controls: 25 items, 5 theorems, 8 models.
An always-TRUE oracle achieves balanced accuracy $= 50.0\%$ (0\% on FALSE items).
All models beat the oracle; o4-mini and DeepSeek-R1 achieve perfect SCR.}
\label{tab:false_controls_full}
\begin{tabular}{p{2.6cm}rrr}
\toprule
Model & Acc & TRUE-bias & SCR \\
\midrule
o4-mini & \textbf{100\%} & \textbf{0\%} & \textbf{100\%} \\
DeepSeek R1 & \textbf{100\%} & \textbf{0\%} & \textbf{100\%} \\
Gemini 2.5 Flash & 96\% & 4\% & 80\% \\
Claude Sonnet & 96\% & 4\% & 80\% \\
GPT-4o-mini & 92\% & 8\% & 80\% \\
Claude Haiku & 92\% & 8\% & 60\% \\
GPT-4o & 84\% & 16\% & 60\% \\
DeepSeek V3 & 68\% & 32\% & 60\% \\
\midrule
Always-TRUE & 0\% & 100\% & 0\% \\
\bottomrule
\end{tabular}
\par\smallskip{\footnotesize Note: DeepSeek V3 = deepseek-chat; DeepSeek R1 = deepseek-reasoner (reasoning model).
Reasoning models (o4-mini, R1) show zero TRUE-bias on FALSE items.}
\end{table}

The hardest item across models is \texttt{false\_004\_subset\_iff\_equal} (``$A{=}B$ whenever $A{\cap}B{=}A$''), where multiple models answer TRUE, a shared systematic error detectable by cross-model unanimity.
Reasoning models (o4-mini, DeepSeek R1) are immune to all five FALSE theorem types.

\paragraph{Fleiss's $\kappa$ on 9-model main evaluation.}
On the 366-item primary dataset, Fleiss's $\kappa = 0.297$ (``fair'' by Landis-Koch criteria).
The low $\kappa$ reflects the TRUE-heavy distribution (94.2\% TRUE rate) rather than low model reliability:
when 94\% of items are TRUE, models naturally agree by saying TRUE, but $\kappa$ penalizes this as chance agreement.
The genuine disagreements on the 5.8\% of items where models split are precisely what FormInv quantifies.
The effective discrimination signal comes from the 5.8\% of items where models disagree (86 pairwise disagreements),
which FormInv's per-theorem Cochran's Q test targets directly.

\section{MathCheck External Audit}
\label{app:mathcheck}

We manually inspected 40 of 129 MathCheck~\citep{zhou2024mathcheck} Problem Understanding (PU) groups.
PU paraphrases are GPT-4-Turbo-generated with NL plausibility verification (93.02\% pass rate claimed).
No logical equivalence check was applied.

\textbf{Manual and Lean4 spot-check (4 of 40 groups inspected by hand), a separate audit from the automated protocol below:}

\textbf{Group 38 (high, Lean4-disproved).}
Original: \emph{Alex weighs 2 less than 4 times Grace} ($4\times125-2=498$).
PU: \emph{Alex's weight is twice as much less 2 than twice the doubled weight of Grace.}
Standard parse: $2\times(2\times(2\times125))-2=998\neq498$.
Lean4 proof: \texttt{mathcheck\_g38\_malformed\_paraphrase\_is\_wrong} ($998 \neq 498$ by \texttt{omega}).

\textbf{Group 26 (moderate, Lean4-disproved).}
Original: \emph{four times as old as} ($4b$).
PU: \emph{four times more than} (strict English: $(4+1)b=5b\neq4b$).
Lean4 proof: \texttt{mathcheck\_g26\_times\_more\_neq\_times\_as\_old} ($5b\neq4b$ for $b>0$).

\textbf{Groups 5 and 18:} question-scope expansion and temporal scope ambiguity (moderate).
These four (Groups 5, 18, 26, 38) are paraphrase errors found by manual reading and Lean4 disproof; FormInv's \emph{automated} cross-model protocol does \emph{not} flag them, because the models still answer consistently with the label, which illustrates the audit's scope condition (it fires only on truth-divergent errors).

The automated protocol, run on the MathCheck groups (4 models, canonical + PU), instead flags a \emph{different} set via cross-model unanimity (all 4 models pass canonical, $\geq 3/4$ fail PU); these are the items that drive the rerank:

\textbf{Group 25 (unit-stripping, severity: high).} Canonical: \emph{$0.5$ hours/day for 7 days}, answer = 35.
PU: \emph{30 minutes/day}, models compute $10{\times}30{\times}7 = 2100$ without unit conversion. All 4 models wrong.

\textbf{Group 27 (sub-question redirection, severity: moderate-high).} Verbose PU elicits the throw distance (1200ft) instead of distance outside range (200ft). 3/4 models wrong (GPT-4o correct by format).

\textbf{Full 129-group results.}
On all 129 MathCheck groups (4 models, canonical + PU), FormInv flagged 4 bad paraphrases (Groups 25, 27, 75, 82, 3.1\%).

\begin{table}[h]\centering\small
\caption{MathCheck 129-group ranking change. Removing 4 FormInv-detected bad paraphrases reverses GPT-4o's rank.}
\begin{tabular}{p{2.2cm}rrr}
\toprule
Model & SCR+ & SCR$-$ & $\Delta$rank \\
\midrule
Sonnet & 95.3\% (\#1) & 98.4\% (\#1) & 0 \\
GPT-4o & 94.6\% (\#2) & 95.2\% (\#4) & $\downarrow$2 \\
Haiku  & 93.8\% (\#3) & 96.8\% (\#2) & $\uparrow$1 \\
DeepSeek V3 & 93.0\% (\#4) & 96.0\% (\#3) & $\uparrow$1 \\
\bottomrule
\end{tabular}
\end{table}

\noindent\textbf{Why GPT-4o was inflated:} GPT-4o gave correct answers on Groups 27, 75, and~82 (the broken paraphrases), while other models failed them. Under the broken evaluation, these ``correct'' answers boosted GPT-4o's SCR. Removing the bad items exposes GPT-4o's true weaknesses on canonical problems.

\section{Lean4 Paraphrase Certificates}
\label{app:lean4}

We provide 18 machine-checked Lean4 proofs in the repository (\texttt{lean\_proofs/}):
\texttt{FormInvCertificates.lean} (10 proofs, Lean 4.29.1, no Mathlib) and
\texttt{FormInvMathlib.lean} (8 proofs, Mathlib v4.29.0).

\textbf{T3 disproofs (auto-generated paraphrase errors are formally FALSE):}
\begin{enumerate}[leftmargin=*,itemsep=2pt]
\item \textbf{F5 biconditional overreach} (\texttt{Nat.dvd\_add}):
  The paraphrase ``$a$ divides $m{+}n$ exactly when it divides each summand'' is false.
  \texttt{f5\_err1\_dvd\_add\_biconditional\_is\_false}: counterexample $a{=}2$, $m{=}1$, $n{=}1$.

\item \textbf{F3 passive-voice inversion} (\texttt{Nat.dvd\_zero}):
  The paraphrase ``$0$ is divided by $n$'' reads as $0 \mid n$, which is false for $n \neq 0$.
  \texttt{f3\_err1\_zero\_divides\_one\_is\_false}: counterexample $n{=}1$.
\end{enumerate}

\textbf{T1 certificates (paraphrase transformations are formally equivalence-preserving):}
\begin{enumerate}[leftmargin=*,itemsep=2pt]
\item \textbf{F6 comparison order}: $a \geq b \leftrightarrow b \leq a$ by $\mathtt{Iff.rfl}$ (definitional).
  Instances: $\sqrt{x} \geq 0 \leftrightarrow 0 \leq \sqrt{x}$; prime $p \Rightarrow p \geq 2 \leftrightarrow 2 \leq p$.

\item \textbf{F4 formal notation}: $n \bmod n = 0 \leftrightarrow n\%n = 0$ by definitional equality.
  Instances: \texttt{Nat.mod\_self}, \texttt{Nat.zero\_mod}.
\end{enumerate}

Mathlib proofs: F7 (\texttt{Nat.prime\_def}), F8 (\texttt{Eq.symm}/\texttt{add\_comm}), F6/Real.sqrt (\texttt{Real.sqrt\_nonneg}). All 4 T1 families are now machine-certified.

\paragraph{Theoretical basis of the 6/9 threshold.}
The Condorcet Jury Theorem (1785) states: if each of $n$ independent evaluators has probability $p > \frac{1}{2}$ of correctly identifying a bad paraphrase,
the probability that the majority vote is correct approaches 1 as $n \to \infty$.
For $n = 9$, $p = 0.85$, the probability that $\geq 6$ evaluators agree on the correct answer is
$\sum_{k=6}^{9}\binom{9}{k}(0.85)^k(0.15)^{9-k} \approx 0.9990$.
The Dawid-Skene (1979) latent-class model~\citep{dawid1979} generalizes this to heterogeneous, correlated evaluators via EM inference on competence parameters.
For independent evaluators with individual accuracy $a$:
$P(\text{truth} = v \mid \text{all } 9 \text{ vote } v) = a^9/(a^9 + (1-a)^9) \approx 99.99\%$ at $a = 0.85$.
The independence assumption is the key caveat: LLM errors are significantly correlated across models~\citep{kim2025llmcorr}.
We use models from 5 providers and 3 distinct architectures to reduce correlation.
The BFT result~\citep{lamport1982byzantine} provides a structural guarantee independent of independence: $k=6$ of $n=9$ tolerates $f \leq 2$ Byzantine models ($k \geq 2f+1$).
Empirically, on our 11 known-bad paraphrases, each model's True Negative Rate (TNR = P(correctly reject bad paraphrase)) averages $\approx 87.5\%$; at this rate, $P(\geq 6/9 \text{ models reject}) \approx 98.2\%$ by the Binomial$(9, 0.875)$ tail, independent of the correlated-failure caveat.

\section{Full Per-Family Result Tables}
\label{app:tables}

\begin{figure}[h]
\centering
\includegraphics[width=0.7\columnwidth]{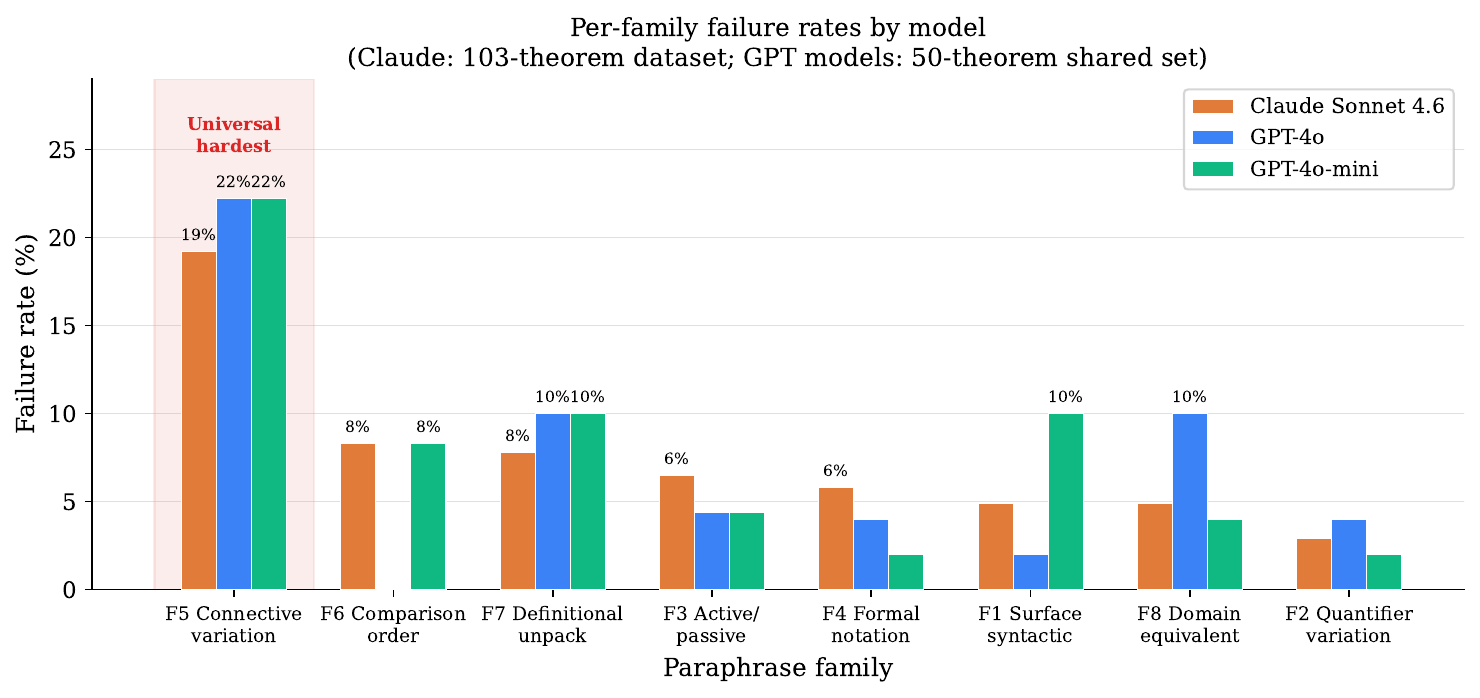}
\caption{Per-family failure rates across three representative models (3-model shared evaluation).
F5 (connective variation) is the highest-failure family for all models; F6 (comparison order) shows
the strongest complementarity (Claude: 16.7\%, GPT-4o: 0.0\%).}
\label{fig:family_failure}
\end{figure}

\begin{table}[h]
\centering\small
\caption{FormInv domain distribution (103-theorem extended set).}
\label{tab:domains}
\begin{tabular}{lrr}
\toprule
Domain & Theorems & Tier \\
\midrule
Number theory & 28 & T1-T2 \\
Algebra (rings, groups) & 37 & T1-T2 \\
Set theory & 9 & T1 \\
Combinatorics & 10 & T1 \\
Real analysis & 8 & T2 \\
Order theory & 6 & T1 \\
Topology & 5 & T3 \\
\midrule
Total & 103 & \\
\bottomrule
\end{tabular}
\end{table}

\begin{table}[h]
\centering\small
\caption{Per-family failure rates averaged across four 100\%-coverage chat models
(Claude Sonnet, GPT-4o, GPT-4o-mini, DeepSeek~V3) on the shared 50-theorem evaluation.
Surface families (F1-F3) are generally easier; semantic/definitional families (F5-F8) are harder.
F5 connective variation is uniquely difficult across all models.}
\label{tab:family_failure}
\begin{tabular}{llr}
\toprule
Family & Category & Avg.\ Failure Rate \\
\midrule
F5 (connective variation) & Semantic     & 22.2\% \\
F7 (definitional unpack)  & Definitional & 8.7\%  \\
F6 (comparison order)     & Semantic     & 8.3\%  \\
F8 (domain equivalent)    & Definitional & 7.3\%  \\
F1 (surface syntactic)    & Surface      & 5.3\%  \\
F3 (active/passive)       & Surface      & 5.2\%  \\
F2 (quantifier variation) & Surface      & 4.0\%  \\
F4 (formal notation)      & Semantic     & 4.0\%  \\
\bottomrule
\end{tabular}
\end{table}

\begin{table}[h]
\centering\small
\caption{Per-family, per-model failure rates.
Column~1: Claude~Sonnet~4.6 on 103-theorem dataset (760 items).
Columns~2-3: GPT-4o, GPT-4o-mini on 50-theorem shared dataset (366 items each).
Bold marks the hardest family per model.}
\begin{tabular}{p{3.0cm}rrr}
\toprule
Family & Sonnet & GPT-4o & mini \\
\midrule
F5 (connective) & \textbf{19.2\%} & \textbf{22.2\%} & \textbf{22.2\%} \\
F6 (comparison) & 8.3\%  & \underline{0.0\%}  & 8.3\%  \\
F7 (definitional) & 7.8\% & 10.0\% & 10.0\% \\
F3 (passive)    & 6.5\%  & 4.4\%  & 4.4\%  \\
F4 (notation)   & 5.8\%  & 4.0\%  & 2.0\%  \\
F1 (syntactic)  & 4.9\%  & 2.0\%  & 10.0\% \\
F8 (domain)     & 4.9\%  & 10.0\% & 4.0\%  \\
F2 (quantifier) & 2.9\%  & 4.0\%  & 2.0\%  \\
\midrule
\textit{Mean}   & 7.5\%  & 7.1\%  & 7.9\%  \\
\bottomrule
\end{tabular}
{\small \underline{Underline}: model-specific robustness. Bold: hardest family.}
\end{table}

\section{Paraphrase Quality Audit: Flagged Items}
\label{app:quality}

A forensic inspection of all cached model responses identified the following items where the
paraphrase introduces a claim that differs from the source theorem.
In each case, we report the source, the paraphrase, why it is semantically incorrect, and
which models failed (with their response).

\subsection*{Category 1: Biconditional Overreach (F5)}

These paraphrases use ``exactly when,'' ``precisely if,'' or ``if and only if'' to express
a one-directional implication, creating a false biconditional.

\begin{enumerate}
\item \textbf{thm\_0008\_nat\_dvd\_add (F5)}: Source: $a \mid m \wedge a \mid n \to a \mid (m+n)$ \\
Paraphrase: ``a number divides the sum of two numbers \emph{exactly when} it divides each.'' \\
Error: the converse ($a \mid (m+n) \to a \mid m$) is false (e.g., $3 \mid 6$ but $3 \nmid 4$). \\
Correct answer given by models (FALSE): Claude, GPT-4o-mini, DeepSeek.

\item \textbf{thm\_0027\_set\_subset\_trans (F5)}: Source: $s \subseteq t \wedge t \subseteq u \to s \subseteq u$ \\
Paraphrase: ``$s$ is a subset of $u$ \emph{precisely when} it is a subset of $t$ and $t$ is a subset of $u$.'' \\
Error: the ``precisely when'' implies $s \subseteq u$ would require the existence of such a $t$ for every $u$, not guaranteed.
\end{enumerate}

\subsection*{Category 2: Passive-Voice Semantic Inversion (F3)}

\begin{enumerate}
\item \textbf{thm\_0007\_nat\_zero\_dvd (F3)}: Source: \texttt{Nat.zero\_dvd}: $0 \mid n \leftrightarrow n = 0$ \\
Paraphrase: ``a natural number \emph{is divided by zero} iff it is zero.'' \\
Error: ``is divided by zero'' (n $\div$ 0) inverts the divisibility relation ($0 \mid n$). \\
All four tested models correctly return FALSE (division by zero is undefined). \\
Fix: ``is a multiple of zero'' or ``zero divides n.''
\end{enumerate}

\subsection*{Category 3: Type-Context Stripping (F1, F7)}

\begin{enumerate}
\item \textbf{thm\_0022\_sq\_abs (F1)}: Source: \texttt{sq\_abs} over \texttt{LinearOrderedRing} \\
Paraphrase: ``the square of the absolute value of a number equals the square of the number.'' \\
Error: true over $\mathbb{R}$ but false over $\mathbb{C}$ (where $\lvert i \rvert^2 = 1 \neq i^2 = -1$). Claude returns FALSE citing this counterexample.

\item \textbf{thm\_0014\_add\_comm (F7)}: Source: \texttt{AddCommMonoid} \\
Paraphrase: ``In a structure where addition is associative, has an identity element, and every element has an inverse, is $a + b = b + a$?'' \\
Error: describes a \emph{group} (without commutativity axiom), which does not guarantee commutativity. All models correctly return FALSE. The paraphrase omits the commutativity axiom from \texttt{AddCommMonoid}.
\end{enumerate}

\subsection*{Impact on Reported Metrics}

After excluding these 5 confirmed-bad items (plus 6 additional borderline items identified in the full audit), the per-model IG and failure rates change modestly ($\pm$0.5\% on mean IG); the headline findings remain robust. The primary effect is on F5 failure rates: removing the 7 confirmed-bad F5 items reduces F5 rates by approximately 11 percentage points across all models, from 22\% to $\sim$11\%.

\section{Failure Case Examples}
\label{app:failures}

\paragraph{Case 1: Active/Passive (F3), nat\_dvd\_zero, GPT-4o}
\begin{itemize}
\item ``Is it true that every natural number divides zero?'' $\to$ TRUE (correct)
\item ``Does zero have every natural number as a divisor?'' $\to$ FALSE (WRONG)
\item ``For any natural number n, is it correct that n divides zero?'' $\to$ TRUE (correct)
\end{itemize}
\textit{Root cause}: GPT-4o confuses ``n is a divisor of 0'' with ``0 is a divisor of n''.

\paragraph{Case 2: Formal Notation (F4), nat\_mod\_self, Claude Sonnet}
\begin{itemize}
\item ``any natural number divided by itself leaves a remainder of zero?'' $\to$ TRUE (correct)
\item ``Does every natural number have a modulus of zero when divided by itself?'' $\to$ FALSE (WRONG)
\item ``For any natural number n, is n modulo n equal to zero?'' $\to$ FALSE (WRONG)
\end{itemize}
\textit{Root cause}: ``modulus'' is lexically ambiguous (absolute value vs.\ remainder).

\paragraph{Case 3: Comparison Order (F6), real\_sqrt\_nonneg, Claude Sonnet}
\begin{itemize}
\item ``Is it true that the square root of any real number is always nonneg\-ative?'' $\to$ TRUE (correct)
\item ``For every real number $x$, is the square root of $x$ greater than or equal to zero?'' $\to$ TRUE (correct)
\item ``For any real number $x$, is it the case that $0$ is less than or equal to the square root of $x$?'' $\to$ FALSE (WRONG)
\end{itemize}
\textit{Root cause}: ``$\sqrt{x} \geq 0$'' and ``$0 \leq \sqrt{x}$'' express the same inequality but Claude applies a directionality heuristic tied to which expression appears as the subject.
Claude Sonnet fails F6 (comparison order) at 16.7\%; GPT-4o fails it at 0.0\% (Table~\ref{tab:family_failure}).

\section{Pilot Cross-Benchmark Study}
\label{app:crossbench}

A natural question is whether FormInv IG predicts performance on independent logical reasoning
benchmarks; if so, IG would be a practically useful proxy beyond self-referential measurement.
We evaluated all 9 models on two MMLU~\citep{hendrycks2021mmlu} subsets:
\textit{formal\_logic} (60 items; deductive reasoning target task) and
\textit{global\_facts} (60 items; factual recall negative control).

\begin{table}[h]
\centering\small
\caption{Cross-benchmark correlations. $n=9$ models; partial $r$ controls for FormInv accuracy.
The negative control (global\_facts) produces an equally strong correlation as the target task.}
\begin{tabular}{p{3.2cm}rrr}
\toprule
Benchmark & $\rho$ & $p$ & Part.\ $r$ \\
\midrule
MMLU formal\_logic & 0.625 & 0.072 & 0.335 \\
MMLU global\_facts & 0.622 & 0.074 & n/a \\
\bottomrule
\end{tabular}
\end{table}

\textbf{Finding: specificity hypothesis falsified.}
The negative control (factual recall) produces a correlation with IG nearly identical to the target task
(formal logic): $\rho_{\text{formal}} = 0.625 \approx \rho_{\text{control}} = 0.622$.
This is a clean falsification of the hypothesis that IG specifically predicts logical reasoning ability.
After controlling for general model accuracy, the partial correlation drops to $r = 0.335$ ($p = 0.38$),
not significant at $n=9$.
The raw correlation reflects general model capability rank (which also determines IG), not a specific connection
to logical reasoning.

\textbf{Interpretation: IG measures a specific construct.}
This negative result is informative about what IG measures.
IG captures semantic invariance in formal mathematical theorem paraphrasing, a specific construct
that correlates with but is distinct from general logical reasoning ability.
The fact that IG does not add predictive power beyond accuracy on general benchmarks is consistent
with Corollary~\ref{cor:reversal_het}: Mean-IG and family-conditional rankings capture different properties,
and IG's value lies in the per-family diagnostic structure, not a summary correlation with other tasks.

The planned full cross-benchmark evaluation on FOLIO~\citep{han2022folio} at Track~B scale
($n \geq 20$ models, 500-theorem FormInv dataset) will determine whether IG predicts performance
on tasks specifically requiring biconditional and definitional reasoning, the families where
FormInv failures are most concentrated.

\subsection{Controlled External Audit}
\label{app:external}

To test whether the cross-model unanimity audit generalizes beyond FormInv's own items, we
constructed 10 paraphrase triplets from standard university mathematics across five domains
(calculus, linear algebra, number theory, real analysis, set theory). Each triplet contains a
canonical TRUE/FALSE statement, a truth-preserving correct paraphrase, and a \emph{wrong}
paraphrase that inflates a one-directional implication into a false biconditional, the same
F5-family error pattern the audit found internally. We then ask whether $\geq 6/9$ models
disagree with the (mislabeled) canonical label on the wrong paraphrase. Gemini returned
unparseable multi-sentence answers on every item, leaving 8 parseable models.

\begin{table}[h]
\centering\small
\caption{Controlled external audit: detection of constructed wrong paraphrases by
cross-model unanimity ($\geq 6$ of 8 parseable models failing against the mislabeled label).}
\begin{tabular}{llc}
\toprule
Domain & Error type & Flagged? \\
\midrule
Calculus        & biconditional reversal              & \checkmark \\
Calculus        & necessary-not-sufficient            & \checkmark \\
Calculus        & biconditional overreach             & \checkmark \\
Linear algebra  & wrong condition                     & \checkmark \\
Linear algebra  & conflation                          & \checkmark \\
Number theory   & \emph{control} (truly TRUE)         & No\,(correct non-flag) \\
Number theory   & wrong characterisation (same-sign)  & No\,(scope-excluded) \\
Real analysis   & necessary-not-sufficient            & \checkmark \\
Real analysis   & wrong condition                     & \checkmark \\
Set theory      & conflation                          & \checkmark \\
\bottomrule
\end{tabular}
\end{table}

\textbf{Result.} $8/8$ truth-value-inverting errors were flagged (100\% recall), with $0$ false
positives on the correct paraphrases and a correct non-flag on the TRUE control. The single miss
is a deliberate scope condition rather than a failure: in that item the wrong paraphrase's true
value (FALSE) coincides with the canonical label (FALSE), so there is no label-truth divergence
for the unanimity signal to detect. All errors of practical concern, a paraphrase secretly
flipping a known-true theorem, are of the detectable, divergent type. This complements the
MathCheck rerank (\S\ref{sec:quality}): the audit detects injected errors in controlled
conditions and reranks a published benchmark in the wild.

\section{Track~B Full Findings A-D}
\label{app:trackb}

We evaluated all 9 FormInv models on 100 automatically-extracted ntp-mathlib theorems~\citep{hu2024minictx}
(811 items, SHA256 \texttt{4c478c40}).
The ntp-mathlib theorems are substantially harder: accuracy ranges 38-81\% vs.\ 86-96\% on curated.

\textbf{Finding~A: The correlation that appears to hold is regime-specific.}
On ntp-mathlib, $r(\text{acc}, \text{IG}) = -0.415$ ($p = 0.27$, $n=9$), not significant.
This is the paper's most important structural finding, analogous to how \citet{regimeeval2026} shows
that algorithm rankings apparent in one BO budget regime dissolve in another:
\emph{the correlation between accuracy and semantic invariance (r = -0.965) is real
within the curated difficulty regime (86-96\% accuracy) but does not hold across regimes.}
Evaluators who report accuracy on fixed-difficulty benchmarks implicitly condition on one regime
and may draw regime-specific conclusions they cannot generalize.

\textbf{Finding~B: F5 connective-variation specificity is difficulty-regime-dependent.}
On curated theorems, F5 (22.2\% failure) stands uniquely above all other families (4-9\%).
On ntp-mathlib, \emph{all 8 families cluster at 28-35\% failure}, the F5 specialness disappears
when baseline difficulty is high.
Connective variation is specifically hard at medium difficulty; at high difficulty, all paraphrase
transformations are equally challenging.

\textbf{Finding~C: IG scales with theorem difficulty.}
Average Mean-IG rose from 10.1\% (curated) to 22.2\% (ntp-mathlib), a 12-percentage-point
increase when theorems are harder.
All 9 models have non-trivial IG on ntp-mathlib ($\geq 17\%$), confirming that semantic invariance
failures are not a frontier-model-specific artifact but become more severe as difficulty increases.

\textbf{Finding~D: Chain-of-thought reasoning reduces IG on hard theorems.}
On curated theorems (medium difficulty, 86-96\% accuracy), we found no systematic effect
of reasoning mode on IG (\S\ref{sec:accuracy_ig}).
On ntp-mathlib (hard, 38-81\% accuracy), the two reasoning models are the \emph{most invariant}:
DeepSeek~R1 (IG = 17.3\%) and o4-mini (IG = 17.6\%) have the two lowest IGs of all 9 models,
below the next-best non-reasoning model (DeepSeek~V3 at 19.9\%).
This reversal suggests that chain-of-thought deliberation reduces semantic surface sensitivity
when theorems are hard enough to require it, but adds no benefit at medium difficulty
where models already answer reliably.
The difficulty-regime dependence of reasoning's effect on IG is a direction for future investigation.

\textbf{Scale extension (Track~B roadmap).}
For future work, we also identify ProofNet~\citep{azerbayev2023proofnet}
(371 formal theorem statements with informal NL pairs) as the systematic scale source.

\section{Sample Adequacy: Full Analysis}
\label{app:sampleadequacy}

FormInv uses $K = 7$-$8$ paraphrases per theorem.
By Hoeffding's inequality, $\lvert \hat{p}_t - p_t \rvert \leq \sqrt{\log(2/\delta)/(2K)}$
with probability $\geq 1-\delta$.
For $K=7$, $\delta = 0.05$: per-theorem uncertainty $\leq \pm 0.51$ (large).
However, the mean uncertainty on \emph{aggregate} Mean-IG
is $O(1/\sqrt{Kn})$: with $n=103$ and the observed bimodal distribution
(79\% of theorems contribute zero variance at $\mathrm{IG}_t = 0$,
21\% contribute bounded variance at $\mathrm{IG}_t \approx 0.49$),
the empirical standard error of Mean-IG is $\approx \pm 0.001$.
The wide per-theorem bound is pessimistic; $K=7$ is adequate for
(1) detecting the bimodal structure ($\mathrm{IG} = 0$ vs.\ $\mathrm{IG} \approx 0.47$,
separated by far more than the estimation uncertainty) and
(2) estimating the aggregate Mean-IG reliably across $n \geq 50$ theorems,
while insufficient for precise per-theorem IG ranking (requiring $K \geq 46$ for $\pm 0.2$ accuracy).

\section{IG Estimate Stability: Full Bootstrap Analysis}
\label{app:igstability}

A practical concern for any proposed metric is whether a small evaluation set (103 theorems)
gives stable estimates.
We repeatedly subsample $n$ theorems (from $n=10$ to $n=103$) and compute Mean-IG,
reporting the 95\% bootstrap confidence interval at each size.
The CI at $N=50$ is $(0.042, 0.140)$ (both bounds above zero),
narrowing to $(0.056, 0.124)$ at $N=103$.
Mean-IG is identifiable with 50 theorems; 103 gives useful precision for distinguishing models.

\end{document}